\documentclass[12pt]{article}

\usepackage[colorlinks]{hyperref}            
\usepackage{color}
\usepackage{graphicx,subfigure,amsmath,amssymb,amsfonts,bm,epsfig,epsf,url,dsfont}
\usepackage{fullpage}
\usepackage[small,bf]{caption}
\usepackage[top=1in,bottom=1in,left=1in,right=1in]{geometry}
\usepackage[sort&compress]{natbib} 
\usepackage{bbm}      
\usepackage{booktabs}
\usepackage{cases}
\usepackage{dsfont}
\usepackage{multirow}
\usepackage{graphicx}

\newtheorem{theorem}{Theorem}
\newtheorem{proposition}{Proposition}
\newtheorem{lemma}{Lemma}

\newtheorem{assumption}{Assumption}
\newenvironment{proof}{\medskip\noindent{\bf Proof.}}{\hfill$\Box$\vspace*{1mm}\medskip}

\newcommand{\hmu}{\widehat{\mu}}

\renewcommand{\phi}{\varphi}

\renewcommand{\P}{\mathbb{P}}
\newcommand{\E}{\mathbb{E}}

\newcommand{\R}{\mathbb{R}}

\newcommand{\KL}{\mathrm{KL}}

\newcommand{\cA}{\mathcal{A}}

\def\ds1{\mathds{1}}
\renewcommand{\epsilon}{\varepsilon}

\newlength{\minipagewidth}
\newcommand{\bookbox}[1]{
\par\medskip\noindent
\framebox[\textwidth]{
\begin{minipage}{\minipagewidth}
{#1}
\end{minipage} } \par\medskip }

\newcommand{\Ber}{\mathop{\mathrm{Ber}}}

\newcommand{\beq}{\begin{equation}}
\newcommand{\eeq}{\end{equation}}

\newcommand{\beqa}{\begin{eqnarray}}
\newcommand{\eeqa}{\end{eqnarray}}

\newcommand{\beqan}{\begin{eqnarray*}}
\newcommand{\eeqan}{\end{eqnarray*}}

\def\ba#1\ea{\begin{align*}#1\end{align*}} 
\def\banum#1\eanum{\begin{align}#1\end{align}} 

\newcommand{\EXP}{\mathbb{E}}

\begin{document}

\title{\bf Bandits with heavy tail}
\author{
S\'ebastien Bubeck\footnote{Department of Operations Research and Financial Engineering, Princeton University} \and 
Nicol{\`o} Cesa-Bianchi\footnote{Dipartimento di Scienze dell'Informazione, Universit{\`a} degli Studi di Milano, Italy} \and
G\'abor Lugosi\footnote{ICREA and Department of Economics, Universitat Pompeu Fabra.
Supported by
the Spanish Ministry of Science and Technology grant MTM2009-09063
and PASCAL2 Network of Excellence under EC grant no.\ 216886.}
}

\date{\today}

\maketitle

\begin{abstract}
The stochastic multi-armed bandit problem is well understood when the
reward distributions are sub-Gaussian. In this paper we examine the
bandit problem under the weaker assumption that the distributions 
have moments of order $1+\epsilon$, for some
$\epsilon \in (0,1]$. 
Surprisingly, moments of order $2$
  (i.e., finite variance) are sufficient to obtain regret bounds
of the same order as under sub-Gaussian reward distributions. 
In order to achieve such regret, we define sampling strategies
based on refined estimators of the mean such as the truncated empirical 
mean, Catoni's $M$-estimator, and the median-of-means estimator.
We also derive matching lower bounds that also show that 
the best achievable regret deteriorates when $\epsilon <1$.
\end{abstract}

\section{Introduction}
In this paper we investigate the classical
stochastic multi-armed bandit problem introduced by \cite{Rob52} and
described as follows: an
agent facing $K$ actions (or bandit arms) selects one arm at every
time step. With each arm $i \in \{1, \hdots, K\}$ there is an
associated probability distribution $\nu_i$ with finite mean
$\mu_i$. These distributions are unknown to the agent. 
At each round $t = 1, \ldots, n$, the agent
chooses an arm $I_t$, and observes a reward drawn from $\nu_{I_t}$
independently from the past given $I_t$. The goal of the agent is to
minimize the {\em regret}
$$R_n = n \max_{i =1, \hdots, K} \mu_i - \sum_{t=1}^n \E\,\mu_{I_t}~.$$

We refer the reader to \cite{BC12} for a survey of the extensive
literature of this problem and
its variations. The vast majority of authors
assume that the unknown distributions
$\nu_i$ are sub-Gaussian, that is, the moment generating 
function of each $\nu_i$ is such that if $X$ is a random variable
drawn according to the distribution $\nu_i$, then
for all $\lambda \geq 0$,
\begin{equation} \label{eq:subgauss}
\ln \E\,e^{\lambda (X-\E X)} \leq \frac{v\lambda^2}{2} \quad\text{and}\quad  
\ln\,\E\,e^{\lambda (\E X - X)} \leq \frac{v\lambda^2}{2}
\end{equation}
where $v>0$, the so-called ``variance factor'' is a parameter that is
usually assumed to be known. In particular, if rewards take values in
$[0,1]$, then by Hoeffding's lemma, one may take $v=1/4$. 
Similarly to the asymptotic bound of~\cite[Theorem~4.10]{Agr95}, this moment assumption was
generalized in \cite[Chapter 2]{BC12} by assuming that there
exists a convex function $\psi : \R_+ \rightarrow \R$ such that, for
all $\lambda \geq 0$,
\begin{equation} \label{eq:psicond}
\ln \E\,e^{\lambda (X-\E X)} \leq \psi(\lambda) \quad\text{and}\quad  \ln\,\E\,e^{\lambda (\E X - X)} \leq \psi(\lambda)~.
\end{equation}
Then one can show that the so-called $\psi$-UCB strategy (a variant of
the basic UCB strategy of \cite{ACF02}) satisfies the following regret
guarantee. Let $\Delta_i=\max_{j =1,
  \hdots, K} \mu_j - \mu_i$, and $\psi^*$ the Legendre-Fenchel
transform of $\psi$, defined by
\[
    \psi^*(\epsilon) = \sup_{\lambda \in \R} \bigl(\lambda \epsilon - \psi(\lambda)\bigr)~.
\]
Then $\psi$-UCB\footnote{More precisely, $(\alpha,\psi)$-UCB with $\alpha=4$.} satisfies
$$R_n \leq \sum_{i \,:\, \Delta_i > 0} \left( \frac{4 \Delta_i}{\psi^*(\Delta_i / 2)} \ln n + 2  \right) .$$
In particular, when the reward distributions are sub-Gaussian, the
regret bound is of the order $\sum_i (\log n)/\Delta_i$, which is known
to be optimal even for bounded reward distributions, see \cite{ACF02}.

While this result shows that assumptions weaker than sub-Gaussian
distributions may suffice for a logarithmic regret, it still
requires the distributions to have finite moment generating function. 
Another disadvantage of the bound above is that the dependence on the 
gaps $\Delta_i$ deteriorates as the tail of the distributions
become heavier. 
In fact, as we show it in this paper, the bound is sub-optimal when
the tails are heavier than sub-Gaussian.

In this paper
we investigate the behavior of the regret when the distributions are
heavy-tailed, and might not have a finite 
moment generating function. 
We show that under significantly weaker assumptions, regret bounds
of the same form as in the sub-Gaussian case may be achieved.
In fact, the only condition we need is that the reward distributions
have a finite variance.
Moreover, even if the variance is infinite but the distributions have
finite moments of order $1+\epsilon$ for some $\epsilon>0$, one may
still achieve a regret logarithmic in the number $n$ of rounds
though the dependency on the $\Delta_i$'s worsens as $\epsilon$ gets
smaller. For instance, for distributions with
moment of order $1+\epsilon$ bounded by $1$ we derive a strategy that satisfies
\[
R_n \leq \sum_{i \,:\, \Delta_i > 0} \left(8 \left(\frac{4}{\Delta_i}\right)^{\frac{1}{\epsilon}} \log n + 5\Delta_i\right)~.
\]
The key to this result is to replace the empirical mean by
more refined robust estimators of the mean and construct 
``upper confidence bound'' strategies. 

We also prove matching lower bounds that show that the proposed strategies
are optimal up to constant factors. In particular the dependency in $1/\Delta_i^{1/\epsilon}$ is unavoidable.

In the following we start by defining a general class of sampling
strategies that are based on the availability of estimators of the
mean with certain performance guarantees. Then we examine 
various estimators for the mean. For each estimator
we describe their performance (in terms of concentration to the mean) 
and deduce the corresponding regret bound.


\section{Robust upper confidence bound strategies}
\label{sec:rucb}

The rough idea behind upper confidence bound ({\sc ucb})
strategies (see \cite{LR85}, \cite{Agr95},
\cite{ACF02}) is that one should choose an arm for which the sum of its
estimated mean and a confidence interval is highest. When
the reward distributions all satisfy the sub-Gaussian condition
(\ref{eq:subgauss}) for a common variance factor $v$, then such a confidence interval is easy to 
obtain. Suppose that at a certain time instance arm $i$ has been
sampled $s$ times and the observed rewards are $X_{i,1},\ldots,X_{i,s}$.
Then the $X_{i,r}$, $r=1,\ldots,s$ are i.i.d.\ random variables 
with mean $\EXP\,X_{i,r}= \mu_i$ and by a simple Chernoff bound, 
for any $\delta \in (0,1)$, the empirical mean $(1/s)\sum_{r=1}^s X_{i,r}$
satisfies, with probability at least $1-\delta$,
\[
   \frac{1}{s} \sum_{r=1}^s X_{i,r} 
     \le \mu_i + \sqrt{\frac{2v\log (1/\delta)}{s}}~.
\]
This property of the empirical mean turns out to be crucial in order
to achieve a regret of optimal order. However, when the sub-Gaussian
assumption does not hold, one cannot expect the empirical mean to 
have such an accuracy. In fact, if one only knows, say, that the variance
of each $X_{i,r}$ is bounded, then the best possible confidence intervals
are significantly wider, deteriorating the performance of standard
{\sc ucb} strategies. (See Appendix \ref{sec:est} for properties
of the empirical mean under distributions of heavy tails.)

The key to successful handling heavy-tailed reward distributions
is to replace the empirical mean with other, more robust, estimators
of the mean. All we need is a performance guarantee like the one
shown above for the empirical mean. 
More precisely, we need a mean estimator with the following property.
\begin{assumption}
\label{ass:est}
Let $\epsilon \in (0,1]$ be a positive parameter and let $c,v$ be positive
constants. Let $X_1,\ldots,X_n$ be i.i.d.\ random variables with finite mean
$\mu$. Suppose that for all $\delta \in (0,1)$ there exists an
estimator $\hmu= \hmu(n,\delta)$ such that, with probability at least
$1-\delta$,
\[
  \hmu   \le \mu + 
  v^{1/(1+\epsilon)}  \left(\frac{c\log (1/\delta)}{n}\right)^{\epsilon/(1+\epsilon)}
\]
and also, with probability at least
$1-\delta$,
\[
  \mu   \le \hmu + 
  v^{1/(1+\epsilon)}  \left(\frac{c\log (1/\delta)}{n}\right)^{\epsilon/(1+\epsilon)}~.
\]
\end{assumption}
For example, if the distribution of the $X_t$ satisfies the
sub-Gaussian condition (\ref{eq:subgauss}), then Assumption~\ref{ass:est}
is satisfied for $\epsilon=1$, $c=2$, and variance factor $v$.
Interestingly, the assumption may be satisfied for significantly more
general distributions by using more sophisticated mean estimators.
We recall some of these estimators in the following subsections, where
we also show how they satisfy Assumption~\ref{ass:est}. As we shall see, the basic
requirement for Assumption~\ref{ass:est} to be satisfied is that the distribution of 
the $X_t$ has a finite moment of order $1+\epsilon$.

We are now ready to define our generalized robust {\sc ucb} strategy, 
described in Figure \ref{fig:1}. We denote by $T_i(t)$ the (random) number 
of times arm $i$ is selected up to time $t$.

\begin{figure}[t]
\bookbox{
{\bf Robust UCB:}

\medskip\noindent
{\em Parameter:} $\epsilon \in (0,1]$, mean estimator $\hmu(t,\delta)$.

\medskip\noindent
For arm $i$, define $\hmu_{i,s,t}$ as the estimate $\hmu(s,t^{-2})$
based on the first $s$ observed values $X_{i,1},\ldots,X_{i,s}$ of
the rewards of arm $i$. Define the index
$$B_{i,s,t}= \hmu_{i,s,t} 
+ v^{1/(1+\epsilon)}  \left(\frac{c\log t^2}{s}\right)^{\epsilon/(1+\epsilon)}~,$$
for $s, t\ge 1$ and $B_{i,0,t}=+\infty$.
\begin{flushleft}
At time $t$, draw an arm maximizing $B_{i,T_i(t-1),t}$.
\end{flushleft}
}
\caption{Robust {\sc ucb} policy.}
\label{fig:1}
\end{figure}

The following proposition gives a performance bound for the robust
{\sc ucb} policy provided that the reward distributions and the mean
estimator used by the policy jointly satisfy Assumption \ref{ass:est}.
Below we exhibit several mean estimators that, 
under various moment assumptions, lead to regret bounds of optimal order.
\begin{proposition}
\label{prop:rucb}
Let $\epsilon \in (0,1]$ and let $\hmu(s,\delta)$ be a mean estimator.
Suppose that the 
distributions $\nu_1, \hdots, \nu_K$ are such that the mean estimator
satisfies Assumption \ref{ass:est} for all $i=1,\ldots,K$.
Then the regret of the Robust {\sc ucb} policy satisfies
\begin{equation} \label{eq:RUCBdistribdep}
R_n \leq \sum_{i \,:\, \Delta_i > 0} \left(2c \left(\frac{v}{\Delta_i}\right)^{\frac{1}{\epsilon}} \log n + 5\Delta_i\right)~.
\end{equation}
Also, if $n$ is such that $\log n \ge \max_i \left(5\Delta_i^{(1+\epsilon)/\epsilon}\big/(2cv^{1/\epsilon})\right)$, then
\begin{equation} \label{eq:RUCBdistribfree}
R_n \leq 
n^{\frac{1}{1+\epsilon}} \bigg(4Kc\log n\bigg)^{\frac{\epsilon}{1+\epsilon}} v^{1/(1+\epsilon)}~.
\end{equation}
\end{proposition}
Note that a regret of at least $\sum_i \Delta_i$ is suffered by any
strategy that pulls each arm at least once. Thus, the interesting
term in (\ref{eq:RUCBdistribdep}) is the one of the order of
$\sum_{i \,:\, \Delta_i > 0} \bigl(v/\Delta_i\bigr)^{\frac{1}{\epsilon}} \log n$.
We show below in Theorem \ref{th:LBdistribdep} that this term
is of optimal order under a moment assumption on the reward distributions.
We also show in Theorem \ref{th:LBdistribdep} that the gap-independent inequality~\eqref{eq:RUCBdistribfree} is optimal up to a logarithmic factor.

\begin{proof}
Both proofs of \eqref{eq:RUCBdistribdep} and \eqref{eq:RUCBdistribfree} rely 
on bounding the expected number of pulls for a suboptimal arm. 
More precisely, in the first two steps of the proof we prove that, 
for any $i$ such that $\Delta_i >0$,
\begin{equation} \label{eq:Tin}
\E\,T_i(n) \leq 
2c \frac{v^{1/\epsilon}}{\Delta_i^{(1+\epsilon)/\epsilon}} \log n +5~.
\end{equation}
To lighten notation, we introduce $u = \left\lceil 
2c \frac{v^{1/\epsilon}}{\Delta_i^{(1+\epsilon)/\epsilon}} \log n \right\rceil$. 
Note that, up to rounding, 
\eqref{eq:Tin} is equivalent to $\E\,T_i(n) \leq u + 4$. 
\newline

\noindent \textbf{First step.}
\newline
We show that if $I_t = i$, then one the following  three inequalities is true:
either
\begin{eqnarray} \label{eq:UCB1}
B_{i^*, T_{i^*}(t-1),t} \leq \mu^*,
\end{eqnarray}
or
\begin{eqnarray} \label{eq:UCB2} 
\hmu_{i,T_i(t-1),t}
 > \mu_i + v^{1/(1+\epsilon)}  \left(\frac{c\log t^2}{T_i(t-1)}\right)^{\epsilon/(1+\epsilon)}
\end{eqnarray}
or 
\begin{eqnarray} \label{eq:UCB3}
T_i(t-1) <  2c \frac{v^{1/\epsilon}}{\Delta_i^{(1+\epsilon)/\epsilon}} \log n.
\end{eqnarray}
Indeed, assume that all three inequalities are false. Then we have
\begin{eqnarray*}
B_{i^*, T_i(t-1),t} & > & \mu^* \\
& = & \mu_i + \Delta_i \\
& \geq & \mu_i + 
2 v^{1/(1+\epsilon)}  \left(\frac{c\log t^2}{T_i(t-1)}\right)^{\epsilon/(1+\epsilon)} \\
& \geq & 
\hmu_{i,T_i(t-1),t}
+ v^{1/(1+\epsilon)}  \left(\frac{c\log t^2}{T_i(t-1)}\right)^{\epsilon/(1+\epsilon)} \\
& = & B_{i,T_i(t-1),t}
\end{eqnarray*}
which implies, in particular, that $I_t \neq i$.
\newline

\noindent \textbf{Second step.}
\newline
Here we first bound the probability that \eqref{eq:UCB1} or \eqref{eq:UCB2} 
hold. 
By Assumption \ref{ass:est} as well as an union bound over the value of 
$T_{i^*}(t-1)$ and $T_i(t-1)$ we obtain
$$\P(\mbox{\eqref{eq:UCB1} or \eqref{eq:UCB2} is true}) 
\leq 2 \sum_{s=1}^t \frac{1}{s^4} \le \frac{2}{t^3}.$$
Now using the first step, we obtain
\begin{eqnarray*}
\E T_i(n) = \E \sum_{t=1}^n \ds1_{I_t=i} & \leq & u + \E \sum_{t=u+1}^n \ds1_{I_t = i \; \mbox{and \eqref{eq:UCB3} is false}} \\
& \leq & u + \E \sum_{t=u+1}^n \ds1_{\mbox{\eqref{eq:UCB1} or \eqref{eq:UCB2} is true}} \\
& \leq & u + \sum_{t=u+1}^n \frac{2}{t^3} \\
& \leq & u + 4~.
\end{eqnarray*} 
This concludes the proof of \eqref{eq:Tin}.
\newline

\noindent \textbf{Third step.}
\newline
Using that $R_n = \sum_{i=1}^K \Delta_i \E T_i(n)$ and \eqref{eq:Tin}, we directly obtain \eqref{eq:RUCBdistribdep}. On the other hand, for \eqref{eq:RUCBdistribfree} we use H{\"o}lder's 
inequality to obtain
 \begin{eqnarray*}
 R_n & = & \sum_{i : \Delta_i>0} \Delta_i (\E T_i(n))^{\frac{\epsilon}{1+\epsilon}} (\E T_i(n))^{\frac{1}{1+\epsilon}} \\
 & \leq & \sum_{i : \Delta_i>0} \Delta_i (\E T_i(n))^{\frac{1}{1+\epsilon}} 
\left(
2c \frac{v^{1/\epsilon}}{\Delta_i^{(1+\epsilon)/\epsilon}} \log n +5
\right)^{\frac{\epsilon}{1+\epsilon}} \\
 & \leq & \sum_{i \,:\, \Delta_i>0} \Delta_i (\E T_i(n))^{\frac{1}{1+\epsilon}} 
\left(
4c \frac{v^{1/\epsilon}}{\Delta_i^{(1+\epsilon)/\epsilon}} \log n
\right)^{\frac{\epsilon}{1+\epsilon}} \\
&  & \text{(by assumption on $n$)} \\
 & \leq &  K^{\frac{\epsilon}{1+\epsilon}}\left(\sum_{i : \Delta_i>0} \E T_i(n)\right)^{\frac{1}{1+\epsilon}}  
(4c)^{\frac{\epsilon}{1+\epsilon}} v^{1/(1+\epsilon)}
(\log n)^{\frac{\epsilon}{1+\epsilon}} 
 \\
&  & \text{(by H\"older's inequality)} \\
 & \leq & n^{\frac{1}{1+\epsilon}} \bigg(4Kc\log n\bigg)^{\frac{\epsilon}{1+\epsilon}} v^{1/(1+\epsilon)}~.
 \end{eqnarray*}
\end{proof}

In the next sections we show how Proposition \ref{prop:rucb} 
may be applied, with different mean estimators, to obtain 
optimal regret bounds for possibly heavy-tailed reward distributions.

\subsection{Truncated empirical mean}

In this section we consider the simplest of the proposed mean estimators,
a truncated version of the empirical mean.
This estimator
is similar to the ``winsorized mean'' and ``trimmed mean'' of Tukey, see
\cite{Bic65}. 

The following lemma shows that if the $(1+\epsilon)$-th raw moment 
is bounded, then the truncated mean satisfies Assumption~\ref{ass:est}.
\begin{lemma} \label{lem:2}
Let $\delta \in (0,1), \epsilon\in (0,1]$, and $u>0$.
 Consider the truncated empirical mean $\hmu_{T}$ defined as
$$\hmu_{T} = \frac{1}{n} \sum_{t=1}^n X_t \ds1_{\left\{|X_t| \leq \left(\frac{u t}{\log (\delta^{-1})}\right)^{\frac{1}{1+\epsilon}}\right\}} .$$
If $\EXP|X|^{1+\epsilon} \le u$, 
then with probability at least $1-\delta$, 
$$\hmu_T \leq \mu + 4 u^{\frac{1}{1+\epsilon}} \left( \frac{\log (\delta^{-1})}{n} \right)^{\frac{\epsilon}{1+\epsilon}} .$$
\end{lemma}
\begin{proof}
Let $B_t = \left(\frac{u t}{\log (\delta^{-1})}\right)^{\frac{1}{1+\epsilon}}$. 
From Bernstein's inequality for bounded random variables, noting that $\E \left( X^2 \ds1_{|X|\leq B} \right) \leq u B^{1-\epsilon}$,
we have, with probability at least $1 - \delta$,
\begin{align*}
\E X - \frac{1}{n} \sum_{t=1}^n X_t \ds1_{|X_t| \leq B_t} 
& = \frac{1}{n} \sum_{t=1}^n \left(\E X - \E \left(X \ds1_{|X| \leq B_t}\right)\right)
+ 
\frac{1}{n} \sum_{t=1}^n \left( \E \left(X \ds1_{|X| \leq B_t}\right) - X_t \ds1_{|X_t| \leq B_t} \right) \\
& =
 \frac{1}{n} \sum_{t=1}^n \E \left( X \ds1_{|X| > B_t}\right)
 +  \frac{1}{n} \sum_{t=1}^n \left( \E \left(X \ds1_{|X| \leq B_t}\right)  - X_t \ds1_{|X_t| \leq B_t} \right) \\
& \leq \frac1{n} \sum_{t=1}^n \frac{u}{B_t^{\epsilon}} + \sqrt{\frac{2 B_n^{1-\epsilon} u \log (\delta^{-1})}{n}} + {\frac{B_n \log (\delta^{-1})}{3 n}}~.
\end{align*}
An easy computation concludes the proof.
\end{proof}

The following is now a straightforward corollary of Proposition \ref{prop:rucb}
and Lemma \ref{lem:2}.
\begin{theorem} 
\label{th:truncated}
Let  $\epsilon \in (0,1]$ and $u>0$. 
Assume that the reward distributions $\nu_1, \hdots, \nu_K$ satisfy
\begin{equation} \label{eq:ass}
\E_{X \sim \nu_i} |X_i|^{1+\epsilon} \leq u \qquad \forall i \in \{1,\hdots, K\}~.
\end{equation}
Then the regret of the Robust-{\sc ucb} policy used with the truncated mean estimator
defined above satisfies
\[
R_n \leq \sum_{i \,:\, \Delta_i > 0} \left(8 \left(\frac{4u}{\Delta_i}\right)^{\frac{1}{\epsilon}} \log n + 5\Delta_i\right)~.
\]
\end{theorem}
When $\epsilon=1$, the only assumption of the theorem above is that 
each reward distribution has a finite variance. In this case the 
obtained regret bound is of the order of $\sum_i (\log n)/\Delta_i$,
which is known to be not improvable in general, even when the rewards are bounded ---note,
however, that the \textsc{kl-ucb} algorithm of \cite{GC11} is never worse than Robust-{\sc ucb} in case of bounded rewards.
We find it remarkable that regret of this  
order may be achieved under the only assumption of finite variance
and one cannot improve the order by imposing stronger tail conditions.

When the variance is infinite but moments of order $1+\epsilon$ are
available, we still have a regret that depends only logarithmically
on $n$. The bound deteriorates slightly as the dependency on
$1/\Delta_i$ is replaced by $1/\Delta_i^{1/\epsilon}$. We show next
that this dependency is inevitable.  
\begin{theorem} 
\label{th:LBdistribdep}
For any $\Delta \in (0,1/4)$, there exist two distributions $\nu_1$ and $\nu_2$ satisfying \eqref{eq:ass} with $u=1$
and with $\mu_1 - \mu_2 = \Delta$, such that the following holds.
Consider an algorithm such that for any two-armed bandit problem satisfying \eqref{eq:ass} with $u=1$ and with arm 2 being suboptimal, one has $\E\,T_2(n) = o(n^a)$ for any $a > 0$. Then on the two-armed bandit problem with distributions $\nu_1$ and $\nu_2$, the algorithm satisfies
\begin{equation} \label{eq:LBdistribdep}
\liminf_{n \to +\infty} \frac{R_n}{\log n} \geq \frac{0.4}{\Delta^{\frac{1}{\epsilon}}}~.
\end{equation}
Furthermore, for any fixed $n$, there exists a set of $K$ distributions satisfying \eqref{eq:ass} with $u=1$ and such that for any algorithm, one has
\begin{equation} \label{eq:LBdistribfree}
R_n \geq 0.01\,K^{\frac{\epsilon}{1+\epsilon}} n^{\frac{1}{1+\epsilon}}~.
\end{equation}
\end{theorem}
\begin{proof}
To prove \eqref{eq:LBdistribdep},
 we take $\nu_1 = (1 - \gamma^{1+\epsilon}) \delta_0 + \gamma^{1+\epsilon} \delta_{1/\gamma}$ with $\gamma = (2 \Delta)^{\frac{1}{\epsilon}}$, and $\nu_2 = (1 + \Delta \gamma - \gamma^{1+\epsilon}) \delta_0 + (\gamma^{1+\epsilon} - \Delta \gamma) \delta_{1/\gamma}$. It is easy to see that $\nu_1$ and $\nu_2$ are well defined, and they satisfy \eqref{eq:ass} with $u=1$ and $\mu_1 - \mu_2 = \Delta$. Now clearly, the two-armed bandit problem with these two distributions is equivalent to the two-armed bandit problem with two Bernoulli distributions with parameters respectively $\gamma^{1+\epsilon}$ and $\gamma^{1+\epsilon} - \Delta \gamma$. Slightly more formally, we could define a new algorithm $\cA'$ that on $\Ber(\gamma^{1+\epsilon}), \Ber(\gamma^{1+\epsilon} - \Delta \gamma)$ behaves equivalently to the original algorithm $\cA$ on $\nu_1$ and $\nu_2$. Therefore, we can use \cite[Theorem 2.7]{Bub10} to directly obtain the following lower bound for $\cA'$,
$$\liminf_{n \to +\infty} \frac{\E\,T_2(n)}{\log n} \geq \frac{1}{\KL\Bigl(\Ber(\gamma^{1+\epsilon} - \Delta \gamma), \Ber(\gamma^{1+\epsilon})\Bigr)}$$
where $\KL$ denotes Kullback-Leibler divergence.
This implies the following lower bound for the original algorithm $\cA$
$$\liminf_{n \to +\infty} \frac{R_n}{\log n} \geq \frac{\Delta}{\KL\Bigl(\Ber(\gamma^{1+\epsilon} - \Delta \gamma), \Ber(\gamma^{1+\epsilon})\Bigr)}~.$$
Equation \eqref{eq:LBdistribdep} then follows directly by using $\KL(\Ber(p), \Ber(q)) \leq \frac{(p-q)^2}{q(1-q)}$ along with straightforward computations.
\newline

The proof of \eqref{eq:LBdistribfree} follows the same scheme. We use the same distributions as above and we consider the multi-armed bandit problem where one arm has distribution $\nu_1$, and the $K-1$ remaining arms have distribution $\nu_2$. Furthermore we set $\Delta=(K / n)^{\frac{\epsilon}{1+\epsilon}}$ for this part of the proof. Now we can use the same proof as for \cite[Theorem 2.6]{Bub10} on the modified algorithm $\cA'$ that runs on the Bernoulli distributions corresponding to $\nu_1$ and $\nu_2$. We leave the straightforward details to the reader.
\end{proof}


\subsection{Median of means}

The truncated mean estimator and the corresponding bandit strategy 
are not entirely satisfactory as they are not translation invariant
in the sense that the arms selected by the strategy may change 
if all reward distributions are shifted by the same constant amount.
The reason for this is that the truncation is centered, quite arbitrarily,
around zero. If the raw moments $\E_{X \sim \nu_i} |X|^{1+\epsilon}$ are
small, then the strategy has a small regret. However, it would be
more desirable to have a regret bound in terms of the centered
moments $\E_{X \sim \nu_i} |X-\mu_i|^{1+\epsilon}$. This is indeed possible
if one replaces the truncated mean estimator by more sophisticated 
estimators of the mean. We show one such possibility, the 
``median-of-means'' estimator in this section. 
In the next section we discuss Catoni's $M$-estimator,
a quite different alternative.

The median-of-means estimator was proposed by \cite{AMS02}. 
The simple idea is to divide the data into various disjoint blocks.
Within each block one calculates the standard empirical mean and
takes a median value of these empirical means.
The next lemma shows that for certain block size the estimator 
has the property required by our robust {\sc ucb} strategy.
\begin{lemma} \label{lem:3}
Let $\delta \in (0,1)$ and $\epsilon \in (0,1]$. 
Let $X_1,\ldots,X_n$ be i.i.d.\ random
variables with mean $\E\,X=\mu$ and centered $(1+\epsilon)$-th moment
$\E|X-\mu|^{1+\epsilon}=u$. 
Let $k=\lfloor 8\log(e^{1/8}/\delta) \wedge n/2 \rfloor$ and
$N=\lfloor n/k \rfloor$.
Let
$$\hat{\mu}_1 = \frac{1}{N} \sum_{t=1}^N X_t \ , \ \hat{\mu}_2 = \frac{1}{N} \sum_{t=N+1}^{2 N} X_t \ , \hdots, \ \hat{\mu}_k = \frac{1}{N} \sum_{t=(k-1)N+1}^{k N} X_t$$ 
be $k$ empirical
mean estimates, each one computed on $N$ data points.
Consider a median $\hmu_{M}$ of these empirical means. 
Then, with probability at least
$1- \delta$,
$$\hmu_M \leq \mu + (12 v)^{\frac{1}{1+\epsilon}} \left( \frac{16 \log (e^{1/8}\delta^{-1})}{n} \right)^{\frac{\epsilon}{1+\epsilon}} .$$
\end{lemma}
\begin{proof}
Let $\eta > 0$ and $Y_{\ell} = \ds1_{\hat{\mu}_{\ell} > \mu + \eta}$ for $\ell \in \{1,\hdots,k\}$. According to equation \eqref{eq:lem1} in the Appendix,
$Y_{\ell}$ has a Bernoulli distribution with parameter
$$p  \leq  \frac{3 v}{N^{\epsilon} \eta^{1+\epsilon}}~.$$
Note that for 
$$\eta = (12 v)^{\frac{1}{1+\epsilon}} \left( \frac{1}{N} \right)^{\frac{\epsilon}{1+\epsilon}}$$
we have $p \leq 1/4$. Thus using Hoeffding's inequality for the tail of
a binomial distribution, we get
$$\P(\hmu_M > \mu + \eta) = \P\left(\sum_{\ell=1}^k Y_{\ell} \geq k/2\right) \leq \exp\left( - 2 k (1/2 - p)^2 \right) \leq \exp(- k/8) = \delta~.$$ 
\end{proof}

The next performance bound is a straightforward consequence 
of Proposition \ref{prop:rucb}
and Lemma \ref{lem:3}.
In some situations it significantly improves on Theorem \ref{th:truncated} 
as the bound depends on the centered moments of order $1+\epsilon$ rather than 
on raw moments.
\begin{theorem} 
\label{th:mm}
Let  $\epsilon \in (0,1]$ and $v>0$. 
Assume that the reward distributions $\nu_1, \hdots, \nu_K$ satisfy
\[
\E_{X \sim \nu_i} |X-\mu_i|^{1+\epsilon} \leq v, \qquad \forall i \in \{1,\hdots, K\}~. 
\]
Then the regret of the
Robust-{\sc ucb} policy used with the median-of-means mean estimator
defined in Lemma \ref{lem:3} satisfies
\[
R_n \leq \sum_{i \,:\, \Delta_i > 0} \left(32 \left(\frac{12v}{\Delta_i}\right)^{\frac{1}{\epsilon}} \log n + 5\Delta_i\right)~.
\]
\end{theorem}

\subsection{Catoni's $M$ estimator}

Finally, we consider an elegant mean estimator introduced by \cite{Cat10}.
As we will see, this estimator has similar performance guarantees as
the median-of-means estimator but with better, near optimal, 
numerical constants. However, we only have a good guarantee in terms of
the variance. Thus, in this section we assume that the variance is
finite and we do not consider the case $\epsilon<1$.

Catoni's mean estimator is defined as follows: Let $\psi: \R\to \R$ be
a continuous strictly increasing function satisfying
\[
     -\log(1-x+x^2/2) \le \psi(x)  \le  \log(1+x+x^2/2)~.
\]
Let $\delta\in (0,1)$ be such that $n>2\log(1/\delta)$ and
 introduce 
\[
\alpha_\delta= \sqrt{ \frac{2\log(1/\delta)}{n(v+\frac{2v\log(1/\delta)}{n-2\log(1/\delta)})}}~.
\]
If $X_1,\ldots,X_n$ be i.i.d.\ random
variables, then Catoni's estimator is defined  
as the 
unique value $\hmu_C=\hmu_C(n,\delta)$ such that
\[
  \sum_{i=1}^n \psi\bigl(\alpha_\delta (X_i-\hmu_C) \bigr) = 0~.
\] 
\cite{Cat10} proves that if $n\ge 4\log(1/\delta)$
and 
the $X_i$ have mean $\mu$ and variance at most
$v$, then, with probability at least $1-\delta$,
\[
\hmu_C \leq \mu + 2 \sqrt{\frac{v\log \delta^{-1}}{n}}
\]
and a similar bound holds for the lower-tail. This bound has the same form
as in Assumption~\ref{ass:est}, though it only holds with the
additional requirement that $n\ge 4\log(1/\delta)$ and therefore
it does not fomally fit in the framework of the robust \textsc{ucb} strategy
as described in Section~\ref{sec:rucb}. However, by a simple 
modification, one may define a strategy that incorporates such a 
restriction. In Figure~\ref{fig:2} we describe a policy based on
Catoni's mean estimator. The policy assumes that there is a known 
upper bound $v$ for the largest variance of any reward distribution.
Then by a simple modification of the proof of Proposition \ref{prop:rucb},
we obtain the following performance bound.

\begin{figure}[t]
\bookbox{
{\bf Modified robust UCB:}

\medskip\noindent
For arm $i$, define $\hmu_{i,s,t}$ as Catoni's mean estimate $\hmu_C(s,t^{-2})$
based on the first $s$ observed values $X_{i,1},\ldots,X_{i,s}$ of
the rewards of arm $i$. Define the index
$$B_{i,s,t}= \hmu_{i,s,t} 
+  \left(\frac{4v\log t^2}{s}\right)^{1/2}~,$$
for $s, t\ge 1$ such that $s\ge 8\log t$ and $B_{i,s,t}=+\infty$ otherwise.
\begin{flushleft}
At time $t$, draw an arm maximizing $B_{i,T_i(t-1),t}$.
\end{flushleft}
}
\caption{Modified robust {\sc ucb} policy.}
\label{fig:2}
\end{figure}
\begin{theorem} 
\label{th:mm}
Let  $v>0$. 
Assume that the reward distributions $\nu_1, \hdots, \nu_K$ satisfy
\[
\E_{X \sim \nu_i} |X-\mu_i|^2 \leq v, \forall i \in \{1,\hdots, K\}. 
\]
Then the regret of the
modified robust {\sc ucb} policy  satisfies
\[
R_n \leq \sum_{i : \Delta_i > 0} \left(\frac{8v\log n}{\Delta_i}
+ 8\Delta_i \log n+  5\Delta_i\right)~.
\]
\end{theorem}
The regret bound has better numerical constants than its analogue based on
the median-of-means estimator. However, a term of the order
$\sum_i \Delta_i \log n$ appears due to the restricted range of validity
of Catoni's estimator.

\section{Discussion and conclusions}
In this work we have extended the \textsc{ucb} algorithm to heavy-tailed stochastic multi-armed bandit problems in which the reward distributions have only moments of order $1+\epsilon$ for some $\epsilon \in (0,1]$. In this setting, we have compared three estimators for the mean reward of the arms: median-of-means, truncated mean, and Catoni's $M$-estimator. The median-of-means estimator gives a regret bound that depends on the central $(1+\epsilon)$-moments of the reward distributions, without need of knowing bounds on these moments. The truncated mean estimator, instead, delivers a regret bound that depends on the raw $(1+\epsilon)$-moments, and requires the knowledge of a bound $u$ on these moments. Finally, Catoni's estimator depends on the central moments like the median-of-means, but it requires the knowledge of a bound $v$ on the central moments, and only works in the special case $\epsilon=1$ (where it gives the best leading constants on the regret). A trade-off in the choice of the estimator appears if we take into account the computational costs involved in the update of each estimator as new rewards are observed. Indeed, while the truncated mean requires constant time and space per update, the median-of-means is slightly more difficult to update, requiring $O(\log \delta^{-1})$ space and $O(\log \log \delta^{-1})$ time per update. Finally, Catoni's $M$-estimator requires linear space per update, which is an unfortunate feature in this sequential setting.


It is an interesting question whether there exists an estimator with the same good concentration properties as the median-of-means, but requiring only constant time and space per update. The truncated mean has good computational properties but the knowledge of raw moment bounds is required. So it is natural to ask whether we may drop this requirement for the truncated mean or some variants of it. Finally, our proof techniques heavily rely on the independence of rewards for each arm. It is unclear whether similar results could be obtained for heavy-tailed bandits with dependent reward processes.

While we focused our attention on bandit problems, the concentration results presented in this work may be naturally applied to other related sequential decision settings. Such examples include the racing algorithms of \cite{maron1997racing}, and more generally nonparametric Monte Carlo estimation, see~\cite{dagum2000optimal} and \cite{domingo2002adaptive}. These techniques are based on mean estimators, and current results are limited to the application of the empirical mean to bounded reward distributions.

\appendix
\section{Empirical mean} \label{sec:est}
In this appendix we discuss the behavior of the standard empirical mean when only moments of order $1+\epsilon$ are available.  We focus on finite-sample guarantees (i.e., non-asymptotic results), as this is the key property to obtain finite-time results for the multi-armed bandit problem.


Let $X, X_1, \hdots, X_n$ be a real i.i.d.\ sequence with finite mean $\mu$. We assume that for some $\epsilon \in (0,1]$ and $v \geq 0$, one has 
$\E |X - \mu|^{1+\epsilon} \leq v$. 
We also denote by $u$ an upper bound on the raw moment of order $1+\epsilon$, that is $\E |X|^{1+\epsilon} \leq u$.

\begin{lemma} \label{lem:1}
Let $\hmu$ be the empirical mean:
$$\hmu = \frac{1}{n} \sum_{t=1}^n X_t~.$$
Then for any $\delta \in (0,1)$, with probability at least $1-\delta$, one has
$$\hmu \leq \mu + \left(\frac{3 v}{\delta n^{\epsilon}}\right)^{\frac{1}{1+\epsilon}}~.$$ 
\end{lemma}
\begin{proof}
Let $\eta, a>0$,
$$\P(\hmu - \mu > \eta) \leq \P\bigl( \exists t \in \{1,\hdots,n\} : |X_t - \mu| > a\bigr) +  \P\left(\frac{1}{n} \sum_{t=1}^n (X_t - \mu) \ds1_{|X_t - \mu| \leq a} > \eta \right).$$
The first term on the right-hand side can be bounded by using a union bound followed by Chebyshev's inequality (for moments of order $1+\epsilon$):
$$\P\bigl( \exists t \in \{1,\hdots,n\} : |X_t - \mu| > a \bigr) \leq n \frac{\E |X - \mu|^{1+\epsilon}}{a^{1+\epsilon}} \le \frac{n v}{a^{1+\epsilon}}~.$$
On the other hand Chebyshev's inequality together with the fact that $\E (X-\mu) \ds1_{|X-\mu| \leq a} = - \E (X-\mu) \ds1_{|X-\mu| > a}$ give for the second term
\begin{align*}
\P&\left(\frac{1}{n} \sum_{t=1}^n (X_t - \mu) \ds1_{|X_t - \mu| \leq a} > \eta \right) \\
& \le \frac{1}{\eta^2} \E\left( \frac{1}{n} \sum_{t=1}^n (X_t - \mu) \ds1_{|X_t - \mu| \leq a} \right)^2 \\
& \leq \frac{\E (X-\mu)^2 \ds1_{|X-\mu| \leq a}}{n \eta^2} +\frac{\left( \E (X-\mu) \ds1_{|X-\mu| \leq a} \right)^2}{\eta^2} \\
& = \frac{\E (X-\mu)^2 \ds1_{|X-\mu| \leq a}}{n \eta^2} +\frac{\left( \E (X-\mu) \ds1_{|X-\mu| > a} \right)^2}{\eta^2}~.
\end{align*}

By applying a trivial manipulation on the first term, and using H{\"o}lder's inequality with exponents $p = 1+\epsilon$ and $q = 1+1/\epsilon$ for the second term, we obtain that the last expression above is upper bounded by
\begin{align*}
\frac{\E |X-\mu|^{1+\epsilon} a^{1-\epsilon}}{n \eta^2} + \frac{\bigl( \E |X-\mu|^{1+\epsilon} \bigr)^{\frac{2}{1+\epsilon}} \bigl( \P(|X-\mu| > a) \bigr)^{\frac{2 \epsilon}{1+\epsilon}}}{\eta^2}
\leq
\frac{v a^{1-\epsilon}}{n \eta^2} + \frac{v^{\frac{2}{1+\epsilon}} v^{\frac{2 \epsilon}{1+\epsilon}}}{\eta^2 a^{2 \epsilon}}~.
\end{align*}
Thus we proved that
$$\P(\hmu - \mu > \eta) \leq \frac{n v}{a^{1+\epsilon}} + \frac{v a^{1-\epsilon}}{n \eta^2} + \frac{v^2}{\eta^2 a^{2 \epsilon}}~.$$
Taking $a = n \eta$ entails 
$$\P(\hmu - \mu > \eta) \leq \frac{2 v}{n^{\epsilon} \eta^{1+\epsilon}} + \left( \frac{v}{n^{\epsilon} \eta^{1+\epsilon}} \right)^2~.$$
Note that if $\frac{v}{n^{\epsilon} \eta^{1+\epsilon}} > 1$ then the bound is trivial, and thus we always have
\begin{equation} \label{eq:lem1}
\P(\hmu - \mu > \eta) \leq \frac{3 v}{n^{\epsilon} \eta^{1+\epsilon}}~.
\end{equation}
The proof now follows by straightforward computations.
\end{proof}

It is easy to see that the order of magnitude of \eqref{eq:lem1} is tight up to a constant factor. Indeed, let $\gamma \in (0,1)$ and consider the distribution $(1 - \gamma^{1+\epsilon}) \delta_0 + \gamma^{1+\epsilon} \delta_{1/\gamma}$. Clearly for this distribution we have $\E |X - \mu|^{1+\epsilon} \leq 1$, so \eqref{eq:lem1} shows that for an i.i.d.\ sequence drawn from this distribution, one has
$$\P(\hmu - \mu > \eta) \leq \frac{3}{n^{\epsilon} \eta^{1+\epsilon}}~.$$
We can restrict our attention to the case where $\eta > n^{- \frac{\epsilon}{1+\epsilon}}$, for otherwise the above upper bound is trivial. Now consider $\gamma = \frac{1}{2 n \eta}$. Note that we have $\mu = \gamma^{\epsilon} = \frac{1}{(2 n \eta)^{\epsilon}} < \eta$ and in particular this implies $1/\gamma = 2 n \eta > n (\eta + \mu)$. From this last inequality and basic computations we get
\begin{align*}
\P(\hmu - \mu > \eta) & \geq \P\bigl( \exists \ i \in \{1, \hdots, n\} : X_i \geq n( \eta + \mu) \bigr) \\
& \ge \P\bigl( \exists \ i \in \{1, \hdots, n\} : X_i = 1 / \gamma \bigr)\\
& = 1 - (1 - \gamma^{1+\epsilon})^n \\
& = 1 - \exp\left(n \ln\left(1 - \frac{1}{(2 n \eta)^{1+\epsilon}} \right) \right)\\
& \geq 1 - \exp\left(- \frac{1}{n^{\epsilon} (2 \eta)^{1+\epsilon}} \right) \\
& = \frac{1}{n^{\epsilon} (2 \eta)^{1+\epsilon}} + o\left(\frac{1}{n^{\epsilon} (2 \eta)^{1+\epsilon}}\right)
\end{align*}
which shows that \eqref{eq:lem1} is tight up to a constant factor for this distribution.
\newline

Clearly, the concentration properties of the empirical mean are much weaker than for the truncated empirical mean or the median-of-means. Indeed, while the dependency on $n$ in the confidence term is similar for the three estimators, the dependency on $1/\delta$ is polynomial for the empirical mean and polylogarithmic for the truncated empirical mean and the median-of-means. As we just showed, this is not an artifact of the proof method, and the empirical mean indeed has polynomial deviations (as opposed to the exponential deviations of the two other estimators). This remark is at the basis of the theory of robust statistics and many approaches to fix the above issue have been proposed, see for example \cite{Hub64, Hub81}. The empirical mean estimator has been previously applied to heavy-tailed stochastic bandits in \cite{LZ11} obtaining polynomial, rather than logarithmic, regret bounds.


\bibliographystyle{plainnat}
\bibliography{newbib}

\end{document}